\def\forarxiv{0}
\def\foraistats{1} 
\def\forneurips{0} 

\if\forarxiv 1
\documentclass{article}
\everypar{\looseness=-1}
\usepackage{fullpage}
\fi 

\if\foraistats 1
\documentclass[twoside]{article}
\usepackage[accepted]{aistats2024}
\fi 

\if\forneurips 1 
\documentclass{article}
\PassOptionsToPackage{numbers, compress}{natbib} 

\usepackage[final]{neurips_2023}
\fi

\usepackage{url}
\usepackage[utf8]{inputenc}
\usepackage{natbib} 
\usepackage{amsmath, amsthm,amssymb}
\usepackage{graphicx}
\usepackage{shortcuts}
\usepackage{xcolor}
\usepackage{cleveref}
\usepackage{mathtools}
\crefname{assumption}{assumption}{assumptions}
\usepackage{appendix}
\usepackage{subcaption}

\newtheorem{theorem}{Theorem}

\newtheorem{lemma}{Lemma}

\newtheorem{remark}{Remark }

\newtheorem{definition}{Definition}
\newtheorem{assumption}{Assumption}
\newtheorem{proposition}{Proposition}
\newcommand{\fitellipsis}[2] %
{\draw let \p1=(#1), \p2=(#2), \n1={atan2(\y2-\y1,\x2-\x1)}, \n2={veclen(\y2-\y1,\x2-\x1)}
    in ($ (\p1)!0.5!(\p2) $) ellipse [x radius=\n2/2+0.8cm, y radius=0.6cm, rotate=\n1];
}

\newcommand{\ts}{\textstyle}

\newcommand\exts[1]{\breve{s}^{[#1]}}

\newcommand{\adv}{\Delta}
\newcommand{\supp}{\rho}
\newcommand{\sparscard}{\vert \supp \vert}
\newcommand{\oracsparscard}{\rho_0}
\newcommand{\threshold}{\tau_0}

\newcommand{\norm}[1]{\lVert#1\rVert}
\newcommand{\lrnorm}[1]{\left\lVert#1\right\rVert}

\newcommand{\covsset}{\mathcal{I}} %

\newcommand{\activeset}{\tilde{\rho}_0}
\newcommand{\largestszerocoords}{\mathcal{T}_0} %
\newcommand{\concentratabilitycoef}{\Lambda_{\min}}
\if \forarxiv 1 
\title{Reward-Relevance-Filtered Linear Offline Reinforcement Learning}
\author{Angela Zhou }

\fi

\usepackage{algorithm}
\usepackage{algorithmic}
\usepackage{fullpage}
\usepackage{color}
\usepackage{tikz}

\usetikzlibrary{shapes,decorations,arrows,calc,arrows.meta,fit,positioning}
\tikzset{
    -Latex,auto,node distance =1 cm and 1 cm,semithick,
    state/.style ={ellipse, draw, minimum width = 0.7 cm},
    point/.style = {circle, draw, inner sep=0.04cm,fill,node contents={}},
    bidirected/.style={Latex-Latex,dashed},
    el/.style = {inner sep=2pt, align=left, sloped}
}

\begin{document}

\if\forneurips 1
\author{%
  Angela Zhou\\
  Department of Data Sciences and Operations\\
  University of Southern California\\
  \texttt{zhoua@usc.edu} 
}
\fi

\if\foraistats 1
\twocolumn[

\aistatstitle{Reward-Relevance-Filtered Linear Offline Reinforcement Learning}

\aistatsauthor{ Angela Zhou}

\aistatsaddress{ University of Southern California \\ Data Sciences and Operations \\ zhoua@usc.edu } ]
\fi 
\if \forarxiv 1
\maketitle
\fi 
\if \forneurips 1
\title{Reward-Relevance-Filtered Linear Offline Reinforcement Learning}
\maketitle
\fi 
\begin{abstract}
This paper studies offline reinforcement learning with linear function approximation in a setting with decision-theoretic, but not estimation sparsity. The structural restrictions of the data-generating process presume that the transitions factor into a sparse component that affects the reward and could affect additional exogenous dynamics that do not affect the reward. Although the minimally sufficient adjustment set for estimation of full-state transition properties depends on the whole state, the optimal policy and therefore state-action value function depends only on the sparse component: we call this \textit{causal/decision-theoretic sparsity}. We develop a method for \textit{reward-filtering} the estimation of the state-action value function to the sparse component by a modification of thresholded lasso in least-squares policy evaluation. %
We provide theoretical guarantees for our reward-filtered linear fitted-Q-iteration, with sample complexity depending only on the size of the sparse component. 
\end{abstract}
\section{Introduction}

\vspace{-10pt}

\if\forneurips 1
Offline reinforcement learning, learning to make decisions from historical data, is necessary in important application areas such as healthcare, e-commerce, and other real-world domains, where randomized exploration is costly or unavailable. It requires certain assumptions such as full observability and no unobserved confounders. This motivates collecting as much information as possible about the environment. On the other hand, common sensing modalities therefore also capture information about the environment that is unaffected by an agent's actions. Given the overall high variance of learning offline, removing such exogenous information can improve policy information and optimization.  

Though various combinations of relevance/irrelevance are possible for rewards and actions, as has been recognized in recent work, most works methodologically impose statistically difficult conditional independence restrictions with variational autoencoders that lack strong theoretical computational/statistical guarantees. Other approaches suggest simpler variable screening, but without discussion of underlying signal strength assumptions, or tradeoffs in downstream estimation and value under potential false negatives/positives, and without guarantees. To bridge between these methods, we focus on a model with linear function approximation, a popular structural assumption in the theoretical literature, and develop methods based on thresholded LASSO regression, connecting classical statistical results to new decision-theoretic notions of sparsity introduced by these causal decompositions of reward/action ir/relevance. In particular, we focus on a particular decomposition: the transitions factor into a sparse component that affects the reward, with dynamics that can affect the next timestep's sparse component and an exogenous component. The exogenous component does not affect the reward or sparse component. A toy example of such a setting is controlling a boat with an image representation of the state environment: actions affect navigation locally and also propagate ripples leaving the boat. Though these ripples evolve under their own dynamics, they themselves do not affect local control of the boat or rewards. Our structural assumptions, though restrictive, still surface what we call ``decision-theoretic, but not estimation sparsity": that is, the minimally sufficient causal adjustment set to predict transition probabilities requires the full state variable, but the optimal policy only depends on the sparse component. 

The contributions of our work are as follows: under our structural assumptions, we develop methodology for filtering out exogenous states based on support recovery via thresholded lasso regression for the rewards, and linear estimation on the recovered support for the $q$ function via least-squares policy evaluation/fitted-Q-iteration (FQI). We prove predictive error guarantees on the $q$ function estimation, and correspondingly on the optimal policy, showing how the optimal policy now depends on the dimensionality of the sparse component, rather than the full ambient dimension. 
\fi 
\if\foraistats 1
Offline reinforcement learning, learning to make decisions from historical data, is necessary in important application areas such as healthcare, e-commerce, and other real-world domains, where randomized exploration is costly or unavailable. It requires certain assumptions such as full observability and no unobserved confounders. This motivates, especially in the era of big data, collecting as much information as possible about the environment into the state variable. On the other hand, common sensing modalities by default capture not only information that can be affected by an agent's actions, but also information about the environment that is unaffected by an agent's actions. For example, in robotics applications, the dynamics of clouds moving in the sky is a separate process that does not affect, nor is affected by, agents' actions, and does not affect agent reward. Given the overall high variance of learning offline, removing such exogenous information can help improve policy information and optimization, while recovering a minimally sufficient state variable for the optimal policy can reduce vulnerability to distribution shifts.  

Though various combinations of relevance/irrelevance are possible for rewards and actions, as has been recognized in a recent work, most works methodologically impose statistically difficult conditional independence restrictions with variational autoencoders that lack strong theoretical computational/statistical guarantees. Other approaches suggest simpler variable screening, but without discussion of underlying signal strength assumptions, or tradeoffs in downstream estimation and value under potential false negatives/positives, and without guarantees. To bridge between these methods, we focus on a model with linear function approximation, a popular structural assumption in the theoretical literature, and develop methods based on thresholded LASSO regression, connecting classical statistical results to new decision-theoretic notions of sparsity introduced by these causal decompositions of reward/action ir/relevance. In particular, we focus on a particular decomposition: the transitions factor into a sparse component that affects the reward, with dynamics that can affect the next timestep's sparse component and an exogenous component. The exogenous component does not affect the reward or sparse component. A toy example of such a setting is controlling a boat with an image representation of the state environment: actions affect navigation locally and also propagate ripples leaving the boat. Though these ripples evolve under their own dynamics, they themselves do not affect local control of the boat or rewards. Our structural assumptions, though restrictive, still surface what we call ``decision-theoretic, but not estimation sparsity": that is, the minimally sufficient causal adjustment set to predict transition probabilities requires the full state variable, but the optimal policy only depends on the sparse component. 

The contributions of our work are as follows: under our structural assumptions, we develop methodology for filtering out exogenous states based on support recovery via thresholded lasso regression for the rewards, and linear estimation on the recovered support for the $q$ function via least-squares policy evaluation/fitted-Q-iteration (FQI). We prove predictive error guarantees on the $q$ function estimation, and correspondingly on the optimal policy, showing how the optimal policy now depends on the dimensionality of the sparse component, rather than the full ambient dimension. 
\fi 

\vspace{-10pt}
\section{Preliminaries}
We consider a finite-horizon Markov Decision Process on the full-information state space comprised of a tuple $\mathcal M = (\mathcal{S}, \mathcal{A}, r, P, \gamma, T)$ of states, actions, reward function $r(s,a)$ , transition probability matrix $P$, $\gamma<1$ discount factor, and time horizon of $T$ steps, where $t=1, \dots, T$. We let the state spaces $\mathcal{S}\subseteq \mathbb{R}^d$ be continuous, and assume the action space $\mathcal{A}$ is finite: $\phi(s,a)$ denotes a (known) feature map. A policy $\pi: \mathcal{S} \mapsto \Delta(\mathcal{A})$ maps from the state space to a distribution over actions, where $\Delta(\cdot)$ is the set of distributions over $(\cdot)$, and $\pi(a\mid s)$ is the probability of taking action $a$ in state $s$. Since the optimal policy in the Markov decision process is deterministic, we also use $\pi(s) \in \mathcal{A}$ for deterministic policies, to denote the action taken in state $s$. The policy and MDP $\mathcal{M}$ induce a joint distribution $P_\pi$ where $P_\pi(a_t \mid s_{0:t}, a_{0:t-1}) = \pi(a_t\mid s_t)$ and  $P_\pi(s_{t+1} \mid s_{0:t}, a_{0:t}) = P(s_{t+1}\mid a_t, s_t)$, the transition probability. 

The value function is $v^\pi_t(s)= \E_\pi[ \sum_{t'=t}^{T} \gamma^{t'-t} r_{t'} \mid s ]$, where $\E_\pi$ denotes expectation under the joint distribution induced by the MDP $\mathcal{M}$ running policy $\pi.$ The state-action value function, or $q$ function is $q^\pi_t(s)= \E_\pi[ \sum_{t'=t}^{T} \gamma r_{t'} \mid s, a]$. These satisfy the Bellman operator, e.g. ${q^\pi_t(s,a) = r(s,a) + \gamma \E[v_{t+1}^\pi(s_{t+1})\mid s,a].}$ The optimal value and q-functions are $v^*, q^*$ correspond to the optimal policy and optimal action, respectively.
We focus on the offline reinforcement learning setting where we have access to a dataset of $n$ offline trajectories, $\textstyle \mathcal{D}=\{(s_t^{(i)}, a_t^{(i)}, s_{t+1}^{(i)})_{t=1}^{T}\}_{i=1}^n$, where actions were taken according to some behavior policy $\pi_b.$ We assume throughout that the underlying policy was stationary, i.e. offline trajectories (drawn potentially from a series of episodes) that are independent.

\paragraph{Linearity}
Throughout this paper, we focus on linear Markov decision processes. Let the feature mapping be denoted $\phi: \mathcal{S} \times \mathcal{A} \mapsto \mathbb{R}^d$. We assume the reward function and value functions are linear in $\phi$. 

\begin{assumption}[Linear MDP]\label{asn-linear-mdp}
Assume that both the rewards and transitions are linear functions (possibly with different parameters): 
\if\foraistats 1
\begin{align*} r_t(s, a)&=\beta_t \cdot \phi(s, a),\\
q_t^\pi(s, a)&=\theta_t^{\pi} \cdot \phi(s, a),  \\
\quad P_t(\cdot \mid s, a)&=\mu_t \phi(s, a), \forall t
\end{align*} 
\fi
\if\forneurips 1
\begin{align*} r_t(s, a)&=\beta_t \cdot \phi(s, a),\;
q_t^\pi(s, a)=\theta_t^{\pi} \cdot \phi(s, a), \;
\quad P_t(\cdot \mid s, a)=\mu_t \phi(s, a), \forall t
\end{align*} 
\fi
\end{assumption}
The theoretical analysis of reinforcement learning typically assumes that the reward function is known, since noise in rewards leads to lower-order terms in the analysis. However, in our setting, we will leverage \textit{sparsity of the rewards} to consider minimal state space representations (and adaptive model selection) which affect first-order terms in the analysis. 

Linear Bellman completeness is the assumption that for any linear function $f(s, a):=\theta^{\top} \phi(s, a)$, the Bellman operator applied to $f(s, a)$ also returns a linear function with respect to $\phi$. (It is an equivalent assumption but generalizes more directly to potential nonlinear settings). 

\begin{definition}[Linear Bellman Completeness]
    the features $\phi$ satisfy the linear Bellman completeness property if for all $\theta \in \mathbb{R}^d$ and $(s, a, h) \in \mathcal{S} \times \mathcal{A} \times[T]$, there exists $w \in \mathbb{R}^d$ such that:
$$
w^{\top} \phi(s, a)=r(s, a)+\gamma \mathbb{E}_{s^{\prime} \sim P_h(s, a)} \max _{a^{\prime}} \theta^{\top} \phi\left(s^{\prime}, a^{\prime}\right) .
$$
\end{definition}
As $w$ depends on $\theta$, we use the notation $\mathcal{T}_h: \mathbb{R}^d \mapsto \mathbb{R}^d$ to represent such a $w$, i.e., $w:=\mathcal{T}_h(\theta)$ in the above equation.
Note that the above implies that $r(s, a)$ is in the span of $\phi$ (to see this, take $\theta=0$ ). Furthermore, it also implies that $q_h^{\star}(s, a)$ is linear in $\phi$, i.e., there exists $\theta_h^{\star}$ such that $q_h^{\star}(s, a)=\left(\theta_h^{\star}\right)^{\top} \phi(s, a)$. 

We let $\supp \in [ d]$ denote an index set. We use the superscript $(\cdot)^\supp$ to denote subindexing a (random) vector by the index set (since time is the typical subscript), i.e. $s^\supp$ is the subvector of state variable according to dimensions $\supp, s^\supp = \{s_k\}_{k \in \supp}$. We also introduce a new notion of \textit{extension} of a subvector $s^\supp$ to the ambient dimension, i.e. $\exts{\supp} = s_k \text{ if } k \in \supp \text{ and } 0 \text{ otherwise}$, which makes it easier, for example, to state equivalence of generic $q$ functions comparing full-dimensional states vs. the extension of sparse subvectors to the full-dimensional space, denoted $\breve{q}$. 

\vspace{-10pt}
    \section{Related work }
Our work is related to sparse offline reinforcement learning, LASSO regression for variable selection, and approaches for leveraging causal structure in reinforcement learning to remove important information. We describe each of these in turn. 

\textbf{Structure in offline reinforcement learning}.
\citep{hao2021sparse} studies LASSO estimation for fitted-q-evaluation and interation, and also suggests thresholded LASSO. Although we also use thresholded LASSO, our method is quite different because we directly impose the sparsity structure induced by reward-relevance into estimation of the $q$ function, because the optimal policy is sparse. An emerging line of work identifies causal decomposition of state variables into reward-relevant/reward-irrelevant/controllable components (or variations thereof) \citep{dietterich2018discovering,wang2022causal,wang2021task,zhang2020invariant,seitzer2021causal,efroni2021provably}. Methodologically, these works regularize representation learning such as with variational autoencoders towards conditional independence (which generally lacks theoretical guarantees) \citep{dietterich2018discovering,wang2022denoised,seitzer2021causal}, or assume specific structure such as block MDPs with deterministic latent dynamics emitting high-dimensional observations \citep{efroni2021provably}, or require auxiliary non-standard estimation \citep{lamb2022guaranteed}. Our model somewhat resembles the exogenous-endogenous decomposition of \citep{dietterich2018discovering}, but swaps cross-dependence of exogenous and endogenous components: this gives \textit{different} conditional independence restrictions directly admits sparse learning. Overall, the main simplification of our model relative to these is that rewards do not depend on the exogenous component. The most methodologically related work is that of \citep{efroni2022sparsity}, which studies sparse partial controllability in the linear quadratic regulator; although they also use thresholded LASSO, they consider online control under a different quadratic cost, focus on controllability (action-relevance), and consider entrywise regression of matrix entries.

\textbf{Variable selection via LASSO}. 
There is an enormous literature on LASSO. We quickly highlight only a few  works on thresholded LASSO. \citep{meinshausen2009lasso} studies model selection properties of thresholded LASSO under a so-called ``beta-min" condition, i.e. an assumed lower bound on the smallest non-zero coefficient and gives an asymptotic consistency result. \citep{zhou2010thresholded} also studies thresholded LASSO, while \citep{van2011adaptive} studies adaptive and thresholded LASSO. For simplicity, we focus on high-probability guarantees under the stronger beta-min condition. But stronger guarantees on thresholded LASSO can easily be invoked instead of the ones we use here. See \citep{buhlmann2011statistics} as well. 

\if\foraistats 1
In a different context, that of single-timestep causal inference, \citep{shortreed2017outcome} proposes the ``outcome-adaptive" lasso which adds a coefficient penalty to estimation of the propensity score based on the inverse-strength of coefficients of the outcome model, to screen out covariates unrelated to both exposure and outcome. We are broadly inspired by the idea to encourage sparsity in one model (in our setting, the $q$-function) based on sparse estimation of another (the reward function). Note, however, that the outcome-adaptive lasso is not applicable to enforce this specific structure. 
\fi

\textbf{Our work. }Even under our simpler model, leveraging classical results from the sparse regression literature sheds light on different approaches that have already been proposed. For example, \citet{wang2022causal} proposes a variable screening method based on independence testing, which performs better for variable selection than a previous regularization-based method \citep{wang2021task}. %
The improvement of thresholding procedures upon regularized LASSO for support recovery is classically well known \citep{buhlmann2011statistics}. The tighter analysis of thresholded lasso also sheds light on implicit signal strength assumptions and tradeoffs of false positives for downstream policy value. 

\if\foraistats 1
Overall, relative to works on exogenous structure in reinforcement learning via representation learning, we connect to a classical literature on sparse regression with provable guarantees. On the other hand, relative to an extensive literature on LASSO, the reinforcement learning setting imposes different decision-theoretic desiderata, such that the optimal policy is sparse (hence q-function) even when from a pure estimation perspective, estimating the transitions are not. 
\fi

\vspace{-10pt}
\section{Structure}
\vspace{-10pt}

\begin{figure}
    \centering
\begin{tikzpicture}
    \node[state] (s) at (0,0) {$s^\supp_0$};
    \node[state] (x) [above =of s] {$s^{\supp_c}_0$};
    \node[state] (a) [right =of s] {$a_0$};
    \node[state] (r) [below =of a] {$r_0$};
    \node[state] (s1) [right =of a] {$s^{\supp}_1$};
    \node[state] (a1) [right =of s1] {$a_1$};
    \node[state] (r1) [below =of a1] {$r_1$};
    \node[state] (x1) [above =of s1] {$s^{\supp_c}_1$};
    \path (s) edge (r);
    \path (s) edge (a);
    \path (a) edge (s1);
    \path (a) edge[dotted] (x1);
    \path (a) edge (r);
    \path (x) edge (a);
    \path (x) edge (x1);
    \path (s) edge (x1);
        \path (s) edge[bend right=30] (s1);
    \path (s1) edge (a1);
\path (s1) edge (r1);
\path (a1) edge (r1);
    \path (x1) edge (a1);
\end{tikzpicture}
    \caption{Reward-relevant/irrelevant factored dynamics. The dotted line from $a_t$ to $s_{t+1}^{\supp_c}$ indicates the presence or absence is permitted in the model.} 
    \label{fig:reward-irrelevance}
\end{figure}
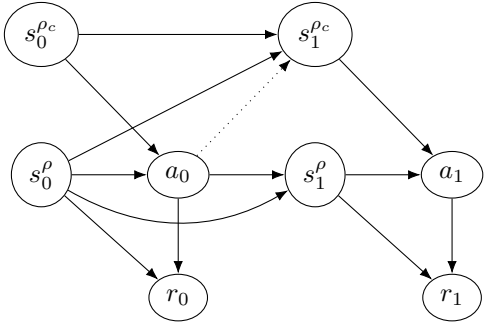

We describe the conditional independence and other restrictions that characterize our filtered reward-relevant model. Let $\supp \subseteq [d]$ denote the supported set of \textit{reward-relevant and endogenous} states. Let $\sparscard$ be the size of the support. 

\begin{assumption}[Blockwise independent design]\label{asn-block-indep}
    $s_{t}^\supp \indep s_{t}^{\supp_c} \mid s_{t-1}, a_{t-1}$
\end{assumption}

\begin{assumption}[Reward-irrelevant decomposition ]\label{asn-reward-relevance}
    Assume that $R(s,a) = R(\tilde{s},a)$ when $s^\supp = \tilde{s}^\supp,$ and that next-time-step endogenous states are independent of prior exogenous states given prior endogenous states and action: 
\begin{equation} s_{t+1}^\supp \indep s_t^{\supp_c} \mid s_t^\supp, a_t \label{eqn-condl-independence} 
\end{equation}
    
\end{assumption}
The conditional independence restriction implies that $
P(s_{t+1}^\supp \mid s_t, a_t) =P(s_{t+1}^\supp \mid s_t^\supp, a_t) $.

Even under these restrictions on the data structure, we can surface a nontrivial qualitative distinction between estimation and decision-making, driven by this causal structure, which we call ``causal sparsity" for short. Although the minimal sufficient adjustment set for estimating the entire-state transitions is the non-sparse union of $s^\rho, s^{\rho_p},$ our next results establish that the optimal decision policy is sparse, and hence our thresholded lasso method depends on the sparse component alone. 

Note that this decomposition differs from the exogenous-endogenous decomposition in \citep{dietterich2018discovering} because our sparse component can affect the exogenous component; but not the other way around -- in our model, the exogenous component does not affect the endogenous component. 

Let $\beta$ be the parameter for the $q$ function, and $\theta$ be the parameter for the reward function. We let $\sigma_r, \sigma_\theta, \sigma_{r + \gamma q}$ denote the subgaussian parameters of the reward-variance, the Bellman-target, and the transitions, respectively.

\paragraph{Interpreting \Cref{asn-reward-relevance}.}
For example, consider linear dynamics (with exogenous noise) in an interacted model, i.e. $s_{t+1}(s,a) = M_a s + \epsilon$ for $M_a \in \mathbb{R}^{d \times d}$. Then $M_a$ is a block matrix and it satisfies \Cref{asn-reward-relevance} if, assuming without loss of generality, that the coordinates are ordered such that the first $\supp$ reward-supported components are first, 
\if\foraistats 1
\begin{align*}&s_{t+1}(s,a) = M_a s + \epsilon, \\
&\text{where } M_a = \begin{bmatrix} M_{a}^{\supp \to \supp } & 0 \\
M_{a}^{\supp \to {\supp_c} } & M_{a}^{{\supp_c} \to {\supp_c} } 
\end{bmatrix}.
\end{align*}
\fi 
\if\forneurips 1
\begin{align*}&s_{t+1}(s,a) = M_a s + \epsilon, \; \text{where } M_a = \begin{bmatrix} M_{a}^{\supp \to \supp } & 0 \\
M_{a}^{\supp \to {\supp_c} } & M_{a}^{{\supp_c} \to {\supp_c} } 
\end{bmatrix}.
\end{align*}
\fi 
In particular, the block matrix $M_{a}^{{\supp_c} \to \supp }=0.$ 

\if\foraistats 1
We can also specify a corresponding probabilistic model. Let $P_a(s_{t+1}\mid s_t)$ denote the $a$-conditioned transition probability, and suppose $P_a(s_{t+1} \mid s_t) \sim N(\mu_a, \Sigma_a),$ and that ${P_a(s_{t+1} \mid s_t)}$ is partitioned (without loss of generality) as $P_a(s_{t+1}^\supp,s_{t+1}^{\supp_c}  \mid s_{t}^\supp,s_{t}^{\supp_c})$. 
Then by \Cref{asn-reward-relevance} $$P_a(s_{t+1}^\supp  \mid s_{t}^\supp) \overset{D}{=} P_a(s_{t+1}^\supp  \mid s_{t}^\supp,s_{t}^{\supp_c}) \sim N(\mu^\supp_a, \Sigma^{\supp,\supp}_a).$$
where the first equality in distribution follows from the conditional independence restriction of \Cref{asn-reward-relevance} and the parameters of the normal distribution follow since marginal distributions of a jointly normal random variable follow by subsetting the mean vector/covariance matrix appropriately.

\begin{remark}
    Similar to previous works studying similar structures, we assume this structure holds. If it may not,
    we could use model selection methods \citep{lee2022oracle}: if we incorrectly assume this structure, we would obtain a completeness violation; so the model selection method's oracle inequalities would apply and be rate-optimal relative to non-sparse approaches. We emphasize that we don't posit this method as a general alternative to general sparsity, but rather as a simple principled approach to estimate in settings with this exogenous structure. 
\end{remark}
\fi 
\vspace{-10pt}
\subsection{Implications for decisions}
We characterize important structural properties under the endogenous-exogenous assumption. Under \Cref{asn-linear-mdp,asn-reward-relevance}, the optimal policy is sparse.

\begin{proposition}[Sparse optimal policies]\label{prop-policy-sparsity}
    When $s^\supp_t = \tilde{s}^\supp_t,$
   $\pi^*_t(s_t) =\tilde{\pi}^*_t(\tilde{s}_t).$
\end{proposition}

\Cref{prop-policy-sparsity} is the main characterization that motivates our method. Even though the estimation of \textit{transitions} are not sparse, the \textit{optimal} q- and value functions are sparse.

Although well-specification/realizability does not imply Bellman completeness of a function class in general, the reward-sparse linear function class is Bellman-complete for $q$ functions as well. Let $\mathcal{F}^\supp_t$ denote the true sparse function classes $\mathcal{F}^\supp_t = \{ \beta \in \mathbb{R}^d \colon \beta_j = 0, j \in \supp \}.$ 
\begin{proposition}[Reward-sparse function classes are Bellman-complete.]\label{prop-bellman-complete}
Let $r^\supp(s,a)$ be the $\supp$-sparse reward function. Let $\breve{q} \in \breve{\mathcal{Q}}$ be the extension of $\supp$-sparse $q$ functions to the full space, i.e. where $\breve{\mathcal{Q}}$ is the space of functions that are $zero$ outside the support $\supp.$ 

Then: 
$\sup_{\breve{q}_{t+1} \in \breve{\mathcal{Q}}_{t+1}} \inf _{q_t \in \breve{\mathcal{Q}}_t}\left\|q_t-\mathcal{T}_t^{\star} q_{t+1}\right\|_{\mu_t}^2 =0 $
\end{proposition}

\vspace{-10pt}

\section{Method }

Based on the posited endogenous-exogenous structure, the sparsity in the linear rewards is the same sparsity pattern as the optimal value function. Notably, the transitions are not sparse unless only regressing on the endogenous states alone. In our method, we first run thresholded LASSO on rewards to recover the sparse support. Then we fit the $q$ function via ordinary least squares as the regression oracle in least-squares policy evaluation/iteration on the estimated support. We describe each of these components in turn; thresholded LASSO, and fitted-Q-evaluation, before describing our specific method in more detail.

Our main estimation oracle of interest is a variant of \textit{thresholded LASSO}, described in \Cref{alg-threshlasso}. {We are not limited to thresholded lasso -- we could develop analogous adaptations of any method that performs well for support recovery. We simply require finite-sample prediction error guarantees, high probability inclusion of the entire support, and bounds on the number of false positives.} %

\begin{minipage}[t!]{0.44\textwidth}
\begin{algorithm}[H]
\caption{Thresholded LASSO }\label{alg-threshlasso}
\begin{algorithmic}[1] 
\STATE{Input: (standardized mean-zero and unit variance) covariate matrix $X$, outcome vector $Y$, from data-generating process where $y=w^\top x + \epsilon.$
}
\STATE{Obtain an initial estimator $w_{\text{init}}$ using the Lasso.}
\STATE{Let ${\hat{\supp}=\{j: w_{ \text { init }}^j>
\threshold
\}}.$}
 \STATE{
Compute ordinary least squares restricted to $\hat\supp$: $$\hat{w}^\rho=(X_{\hat\supp_k}^T X_{\hat\supp})^{-1} X_{\hat\supp}^T Y.$$}

\end{algorithmic}
\end{algorithm}
\end{minipage}
\hfill
\begin{minipage}[t!]{0.56\textwidth}
\begin{algorithm}[H]
\caption{Reward-Filtered Fitted Q Iteration }\label{alg-rewardsparse-fqi}
\begin{algorithmic}[1]
    \STATE{At timestep $t=T:$ 
    
    Run thresholded LASSO (\Cref{alg-threshlasso}) on $r_{T}$ and obtain sparse support $\hat\supp_{T}$.
    
    $\pi^*_{T}(s^{\hat\supp_{T}})=\arg\max_a q_{T}(s^{\hat\supp_{T}},a).$
    
    }
    \FOR{ timestep $t=T-1, \ldots, 1$}
    \STATE{Run thresholded LASSO (\Cref{alg-threshlasso}) on $r_t$.
    
    Obtain sparse support $\hat{\supp}_t$.}
    \STATE{Compute Bellman target $$\textstyle y_t=r_t+\gamma \E_{\pi^{*,\supp}_{t+1}}[
 q_{t+1}(s_{t+1},a_{t+1})].$$
    }
    \STATE{Fit Bellman residual restricted to $\hat\supp_t$.
    $$ \textstyle 
\widetilde{\beta}_t\in \underset{\beta \in \mathbb{R}^p}{\arg\min}\{ \frac{1}{2} \E_{n}[ ( \beta^\top \phi_t- y_t)^2 ] \colon \beta_j=0 ,\;j \in \hat{\supp}^c_t \}
$$}
\STATE{   $\pi^*_{t}(s^\supp)=\arg\max_a q_{t}(s^\supp,a).$}
    \ENDFOR

\end{algorithmic}
\end{algorithm}
\end{minipage}

\paragraph{Fitted-Q-Iteration }

Linear fitted-q-evaluation, equivalent to offline least-squares policy evaluation, \citep{ernst2006clinical,le2019batch,nedic2003least,duan2020minimax}, and fitted-Q-iteration \citep{chen2019information,duan2021risk} successively approximate $\hat{q}_t$ at each time step by minimizing an empirical estimate of the Bellman error:
\if\foraistats 1
\begin{align*}
    y_t(q) &\coloneqq r_t +  \max_{a'} \left[q(s_{t+1},a')\right],\\
q_t(s,a) &=
\E[ y_t(q_{t+1})  | s_t = s, a_t = a],\\
    \hat{q}_t&
    \in \arg\min_{q_t \in \mathcal Q }{\E}_{n,t}[ (y_t(\hat{q}_{t+1})-q_t(s_t,a_t))^2].\label{eqn-nomoutcomes}%
\end{align*} 
\fi
\if\forneurips 1
\begin{align*}
    y_t(q) &\coloneqq r_t +  \max_{a'} \left[q(s_{t+1},a')\right], \; 
q_t(s,a) =
\E[ y_t(q_{t+1})  | s_t = s, a_t = a],\\
    \hat{q}_t&
    \in \arg\min_{q_t \in \mathcal Q }{\E}_{n,t}[ (y_t(\hat{q}_{t+1})-q_t(s_t,a_t))^2].\label{eqn-nomoutcomes}%
\end{align*} 
\fi
The Bayes-optimal predictor of $y_t$ is the true $q_t$ function, even though $y_t$ is a stochastic approximation of $q_t$ that replaces the expectation over the next-state transition with a stochastic sample thereof (realized from data). 

\paragraph{Our method}

Our algorithm, described in \Cref{alg-rewardsparse-fqi}, is a natural modification of these two ideas. At the last timestep, we simply run thresholded lasso on the rewards and set the optimal policy to be greedy with respect to the sparsely-supported reward. At earlier timesteps, we first run thresholded lasso on the rewards and recover an estimate of the sparse support, $\rho_t$. Then, we fit the Bellman residual $(r_t+\E_{\pi^{*,\supp}_{t+1}}[
 q_{t+1}(s_{t+1},a_{t+1})] - q_t(s_t,a_t))^2$ over linear functions of $\phi_t$ that are supported on $\rho_t$. That is, we use the sparse support estimated from rewards only in order to sparsely fit the $q_t$ function. Again we set the optimal policy to be greedy with respect to the sparse $q_t$ function.

\if\foraistats 1 
\paragraph{Why not simply run thresholded LASSO fitted-Q-iteration? }
Lastly, we provide some important motivation by outlining potential failure modes of simply applying thresholded lasso fitted-Q-iteration (without specializing to the endogenous-exogenous structure here). 
The first iteration (last timestep), $q_{T}= R_{T}$. %
So thresholded regression at last timestep is analogous to thresholded reward regression. Note that if reward regression succeeds at time ${T}$, then we are integrating a dense measure against the sparse function $V_{T}$. On the other hand, mistakes in time ${T}$ will get amplified (i.e. upboosted as ``signal" by the dense transition measure). Our reward-thresholded LASSO will not accumulate this error based on the structural assumptions. Without these structural assumptions, it would be unclear whether the rewards are truly dense or whether the dense transitions are amplifying errors in support recovery on the rewards. 
\fi 
\vspace{-10pt}

\section{Analysis}
\vspace{-10pt}

We show a predictive error bound, approximate Bellman completeness under the strong-signal support inclusion of thresholded LASSO, and improvement in policy value. 
The main technical contribution of our work is the finite-sample prediction error bound for the reward-thresholded fitted-Q-regression. Typical prediction error analyses of thresholded lasso do not directly apply to our setting, where we recover the support from the reward and apply it directly to the $q$-function estimation. The key observation is that the two regressions share covariance structure and some outcome structure in part. Given this result on the finite-sample prediction error and high-probability inclusion of high-signal sparse covariates, since fitted-Q-evaluation analysis uses prediction bounds on regression in a black-box way, we immediately obtain results on policy value. See \citep{buhlmann2011statistics, ariu2022thresholded,zhou2010thresholded} for discussion of analysis of thresholded LASSO.

\subsection{Preliminaries: standard convergence results for thresholded LASSO}
 Let $x_t= \phi(s_t,a_t)$ denote regression covariates, with $y_t$ the Bellman residual; in this statement we drop the timestep for brevity and let $(X,Y)$ denote the data matrix and outcome vector, e.g. at a given timestep concatenated over trajectories. Our first assumption is that transition probabilities are time-homogeneous. 
\begin{assumption}\label{asn-time-homogeneous}
    Time-homogeneous transitions. 
\end{assumption}

Next we define problem-dependent constants used in the analysis, assumptions, and statements. 
\begin{definition}[Problem-dependent constants.]
For $a\geq 0$, define 
\if\foraistats 1
\begin{align} \lambda_{\sigma, a, d}&\coloneqq \sigma \sqrt{1+a} \sqrt{2 \log p / n}, \\
\mathcal{E}_a & \ts 
\coloneqq \left\{\epsilon:\left\|X^T \epsilon / n\right\|_{\infty} \leq \lambda_{\sigma, a, p}\right\}.
\end{align} 
\fi
\if\forneurips 1
\begin{align} \lambda_{\sigma, a, d}&\coloneqq \sigma \sqrt{1+a} \sqrt{2 \log p / n}, 
\mathcal{E}_a \ts 
\coloneqq \left\{\epsilon:\left\|X^T \epsilon / n\right\|_{\infty} \leq \lambda_{\sigma, a, p}\right\}.
\end{align} 
\fi 
$\lambda_{\sigma, a, d}$ bounds the maximum correlation between the noise and covariates of $X$ and $\mathcal{E}_a$ is a high probability event where ${P}\left(\mathcal{E}_a\right) \geq 1-\left(\sqrt{\pi \log p} p^a\right)^{-1}$ when $X$ has column $\ell_2$ norms bounded by $\sqrt{n}$.
Let $\oracsparscard \leq s$ be the smallest integer such that:
\if\foraistats 1
    $$\ts
\sum_{i=1}^p \min \left(\beta_i^2, \lambda^2 \sigma^2\right) \leq \oracsparscard \lambda^2 \sigma^2.%
$$
\fi
\if\forneurips 1
$\ts
\sum_{i=1}^p \min \left(\beta_i^2, \lambda^2 \sigma^2\right) \leq \oracsparscard \lambda^2 \sigma^2.%
$
\fi 

Let $\mathcal{T}_0$ denote the largest $\oracsparscard$ coordinates of $\beta$ in absolute values. 
 Define an active set of strong-signal coordinates, for which we would like to assure recovery, and $\activeset \subseteq \largestszerocoords \subset \supp$:
 \begin{equation}\activeset=\left\{j:\left|\beta_j\right|>\lambda \sigma\right\}, \label{eqn-def-activeset}\end{equation} 
\end{definition} 

We assume standard restricted-eigenvalue conditions and beta-min conditions for support inclusion results. 
\begin{assumption}[Restricted Eigenvalue Condition $RE(\sparscard, k_0, X)$ (Bickel et al., 2009)]\label{asn-rec} 
Let $X$ be the data matrix. Define
$$
\frac{1}{\kappa\left(\sparscard, k_0\right)} \triangleq \min_{\substack{J_0 \subseteq\{1, \ldots, d\} \\
\left|J_0\right| \leq \sparscard}} \min _{\left\|v_{J_0}\right\|_1 \leq k_0\left\|v_{J_0}\right\|_1} \frac{\|X v\|_2}{\sqrt{n}\left\|v_{J_0}\right\|_2}.$$
For some integer $1 \leq \sparscard \leq d$ and a number $k_0>0$, it holds for all $v \neq 0$,
\if\foraistats 1
\begin{align*} 
&{\kappa\left(\sparscard, k_0\right)^{-1}}>0, \\
&  \Lambda_{\min }(2 \sparscard) \coloneqq \underset{v \neq 0, \norm{v}_0 \leq 2 \sparscard
}{\min} \frac{\|X v\|_2^2}{n\|v\|_2^2}>0 , \\
& \Lambda_{\min }(2 \sparscard) \coloneqq \underset{v \neq 0, \norm{v}_0 \leq 2\sparscard 
}{\max} \frac{\|X v\|_2^2}{n\|v\|_2^2}>0.
\end{align*} 
\fi 
\if\forneurips 1
\begin{align*} 
&\textstyle {\kappa\left(\sparscard, k_0\right)^{-1}}>0,   \;\Lambda_{\min }(2 \sparscard) \coloneqq \underset{v \neq 0, \norm{v}_0 \leq 2 \sparscard
}{\min} \frac{\|X v\|_2^2}{n\|v\|_2^2}>0 , 
\; \Lambda_{\min }(2 \sparscard) \coloneqq \underset{v \neq 0, \norm{v}_0 \leq 2\sparscard 
}{\max} \frac{\|X v\|_2^2}{n\|v\|_2^2}>0.
\end{align*} 
\fi 
\end{assumption}

The restricted eigenvalue condition of \Cref{asn-rec} is one of the common assumptions for LASSO. It corresponds to assuming well-conditioning of the matrix under sparse subsets. It also ensures that the behavior policy provides good coverage over relevant features; indeed it characterizes coverage for linear function approximation \citep{duan2020minimax}.

\begin{assumption}[Beta-min condition on strong signals]\label{asn-betamin}
 $\beta_{\min,\activeset}\coloneqq
    \min_{j \in\activeset}\left|\beta_j\right|>\lambda \sigma_r$.
\end{assumption}
\Cref{asn-betamin} is a signal-strength condition, that the smallest coordinate of the active set is separated from the threshold defining the active set. This prevents knife-edge situations where a relevant coordinate is not recovered (but is also of irrelevant signal strength). Analogous assumptions are generally required to show support inclusion. \Cref{asn-betamin} is somewhat milder; instead imposing a stronger version would give correspondingly stronger recovery results. 

Under these assumptions, our main result is a prediction error bound on $q$-function estimation under reward-thresholded lasso, under given rate conditions on threshold and regularization strength of initial lasso.
\begin{theorem}[Prediction error bound for reward-thresholded LASSO]\label{thm-q-predictionerror}
Suppose \Cref{asn-block-indep,asn-linear-mdp,asn-rec,asn-reward-relevance,asn-time-homogeneous,asn-betamin}. 
Suppose \Cref{asn-rec}, $\operatorname{RE}\left(\oracsparscard, 4, X\right)$ holds with $\kappa\left(\oracsparscard, 4\right)$. 

Let $\beta_{\text {init }}$ be an optimal solution to $\textrm{LASSO}(\phi, r; \lambda_n)$, e.g. lasso regression of rewards on features, with $\lambda_n \geq \frac{\norm{X\epsilon_\theta}_\infty}{n}$.
Suppose that for some constants $\breve{D}_1 \geq D_1$, and for $D_0(\Lambda_{\max}, \Lambda_{\min}, \sparscard, \oracsparscard), D_1(\Lambda_{\max}, \Lambda_{\min}, \sparscard, \oracsparscard)$ specified in the appendix, it holds that
$
\beta_{\min , \activeset} \geq D_0 \lambda_n \sigma \sqrt{\oracsparscard}+\breve{D}_1 \lambda_n \sigma.
$
Choose {threshold} $\threshold=C \lambda \sigma \geq 2 \sqrt{1+a} \lambda \sigma$, for some constant $C \geq D_1$. Let $\mathcal{I}$ be the recovered support on $\beta_{\text{init}}.$
$$
\mathcal{I}=\left\{j:\left|\beta_{j, \text {init }}\right| \geq \threshold\right\}, \text { where } \threshold \geq \breve{D}_1 \lambda \sigma .
$$%
Then on $\mathcal{E}_a$, 
\if\foraistats 1
\begin{align*}
    \activeset \subset \covsset, \; |\covsset| \leq 2\oracsparscard, \text{ and } \left|\covsset \cap \largestszerocoords^c\right| \leq \oracsparscard
\end{align*}
\fi
\if\forneurips 1 
$   \activeset \subset \covsset, \; |\covsset| \leq 2\oracsparscard, \text{ and } \left|\covsset \cap \largestszerocoords^c\right| \leq \oracsparscard.$
\fi 
And, with high probability we have predictive error bounds: 
    \begin{align*}
&\textstyle 
\frac 1n  \norm{X \hat{\theta}-X \theta^* }_2^2 \leq 
4 \frac{\sigma^2_q (\vert \covsset \vert (1 + 468 \log (2d)) +2 ( 1 + 2 \sqrt{\vert \covsset \vert}) 
 }{n}.
    \end{align*}
\end{theorem}
Given this ``fast rate" on the prediction error of the reward-thresholded LASSO, we obtain a bound on the policy error of the fitted-Q-iteration procedure that depends primarily on the \textit{sparsity} (up to constant factors) rather than the potentially \textit{high-dimensional state}. The analysis is standard, given the result we prove above specialized for our method. Note that we did not attempt to optimize problem-independent constants in our analysis.

Before we do so, we show how the thresholded procedure also quantifies an important structural restriction for policy evaluation/optimization: \textit{(approximate) Bellman completeness}, which states that the Bellman operator is approximately closed under the regression function class. Although \Cref{prop-bellman-complete} establishes that the class of linear functions restricted to the sparse component is Bellman complete, in practice, thresholding noisy estimates may lead to false positives and false negatives. Our previous analysis establishes that these are of controlled magnitude due to the choices of thresholding and regularization parameter. This also implies that the \textit{misspecification bias due to finite-sample estimation} is also vanishing in $n$ at the same rate, stated in the following proposition on approximate instance-dependent Bellman completeness.
\begin{proposition}[Bound on Bellman completeness violation under approximate recovery]\label{prop-apx-completeness}
With high probability, under $\mathcal{E}_a,$
    $$ \textstyle 
\underset{q_{t+1} \in \mathcal{Q}_{\covsset, \rho\setminus\activeset \not\subseteq \covsset}}{\sup} \underset{q_t \in \mathcal{Q}_{\covsset, \rho\setminus\activeset \not\subseteq \covsset}}{\inf}
     \|q_t-\mathcal{T}_t^{\star} q_{t+1}\|_{\mu_t}^2
= O_p(n^{-1}).
     $$
\end{proposition}

With these results, we can establish a finite-sample bound on the policy value under \Cref{alg-rewardsparse-fqi}.
\begin{theorem}[]\label{thm-fqi-bound}
Suppose \Cref{asn-block-indep,asn-linear-mdp,asn-rec,asn-reward-relevance,asn-time-homogeneous,asn-betamin}. 
\if\foraistats 1 
\begin{align*}
&V_1^*(s_1) - V_1^\pi(s_1) \\
&\leq  2T \sqrt{\frac{\concentratabilitycoef \sigma^2_q ( 938 \sparscard  \log (2d) +2 ( 1 + 2 \sqrt{\sparscard}) 
 }{n}} + o_p(n^{-\frac 12}).
\end{align*}
\fi 
\if\forneurips 1 
\begin{align*} \textstyle 
&\textstyle  V_1^*(s_1) - V_1^\pi(s_1) \leq  2T \sqrt{\frac{\concentratabilitycoef \sigma^2_q ( 2 \sparscard (1 + 468 \log (2d)) +2 ( 1 + 2 \sqrt{\sparscard}) 
 }{n}} + o_p(n^{-\frac 12}).
\end{align*}
\fi
\end{theorem}
The result follows straightforwardly given our predictive error bound and standard analysis of fitted-Q-iteration. 
This sample complexity result improves upon prior work since it now depends on the underlying sparsity rather than the full ambient dimension.

\vspace{-10pt}
\section{Experiments}
\vspace{-10pt}

\begin{figure*}
    \centering
\begin{subfigure}{0.33\textwidth}
    \includegraphics[width=\textwidth]{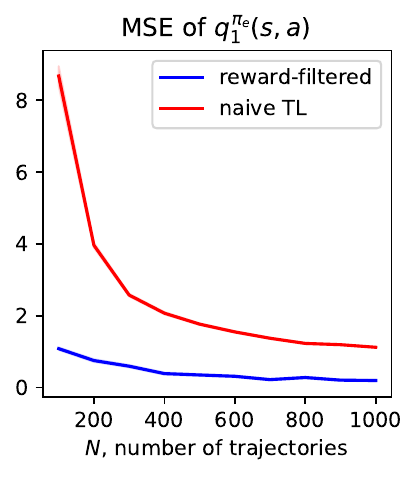}
    \caption{MSE of $\hat q^{\pi_e}_1$ under naive thresholding of q-estimation vs. reward-filtered evaluation}
    \label{subfig:mse}
\end{subfigure}\begin{subfigure}{0.33\textwidth}
    \includegraphics[width=\textwidth]{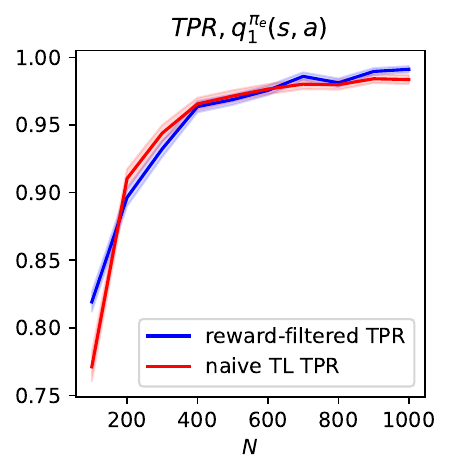}
    \caption{True positive rate, recovered support of $\hat q^{\pi_e}_1$}
    \label{subfig:tpr}
\end{subfigure}\begin{subfigure}{0.33\textwidth}
    \includegraphics[width=\textwidth]{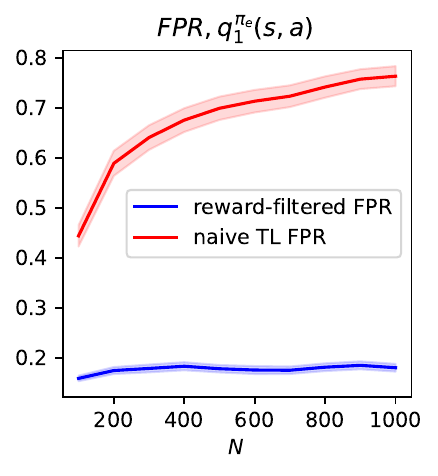}
    \caption{False positive rate, recovered support of $\hat q^{\pi_e}_1$}
    \label{subfig:fpr}
\end{subfigure}
\end{figure*}
We first consider a simulated setting to validate the method. Our primary comparison is with thresholded LASSO regression for fitted-Q-evaluation. This highlights the benefit of tailoring estimation for the inductive bias. \if\forneurips 1 
The figures are in \Cref{subfig:fpr,subfig:mse,subfig:tpr} but due to space constraints, see the Appendix for experimental details. We benchmark against naive thresholded LASSO estimation and improve estimation error while controlling false positives incorrectly included by the baseline.
\fi
\if\foraistats 1 
In the data-generating process, we first consider $\vert \mathcal{S} \vert = 50, \sparscard = 10$, and $\mathcal{A} = \{ 0, 1\}$. The reward and states evolve according to $$r_t(s,a) = \beta^\top \phi_t(s,a) + \epsilon_r, \;\; s_{t+1}(s,a) = M_a s + \epsilon_s.$$ Recalling that $M_a = \begin{bmatrix} M_{a}^{\supp \to \supp } & 0 \\
M_{a}^{\supp \to {\supp_c} } & M_{a}^{{\supp_c} \to {\supp_c} } 
\end{bmatrix},$ we generate the coefficient matrix with independent  normal random variables $\sim N(0.2, 1)$. (Note that the nonzero mean helps ensure the beta-min condition). The zero-mean noise terms are normally distributed with standard deviations $\sigma_s = 0.4, \sigma_r = 0.6.$ In the estimation, we let $\phi(s,a)$ be a product space over actions, i.e. equivalent to fitting a $q$ function separately for every action. 

We first show experiments for policy evaluation in the main text due to space constraints. Fitted-Q-evaluation is similar to fitted-Q-iteration, but replaces the max over q functions with the expectation over actions according to the next time-step's policy. See the appendix for additional experiments for policy optimization specifically. We compare our reward-filtered estimation using \Cref{alg-rewardsparse-fqi} with naive thresholded lasso, i.e. thresholding lasso-based estimation of q-functions alone in \Cref{subfig:fpr,subfig:mse,subfig:tpr}. (We average the $q$ function over actions; results are similar across actions). The behavior and evaluation policies are both (different) logistic probability models in the state variable, with the coefficient vector given by (different) random draws from the uniform distribution on $[-0.5,0.5].$ We average over 50 replications from this data generating process and add standard errors, shaded, on the plot. The first plot, \Cref{subfig:mse}, shows the benefits in mean-squared error estimation of the q-function $q_1^{pi_e}(s,a),$ relative to the oracle $q$ function, which is estimated from a separate dataset of $n=20000$ trajectories. The reward-filtered method achieves an order of magnitude smaller mean-squared error for small sample sizes, with consistent improvement over thresholded LASSO estimation on the $q$ function alone. Next in \Cref{subfig:tpr} we show the true positive rate: both methods perform similarly in including the sparse component the recovered support. But the last plot of \Cref{subfig:fpr} shows that the naive thresholded lasso method includes many exogenous variables that are not necessary to recover the optimal policy, while the false positive rate for the reward-filtered method is controlled throughout as a constant fraction of the sparsity. Overall this simple simulation shows the improvements in estimation of the $q$ function (which translate down the line to improvements in decision-value) under this special structure.  
\fi 
\paragraph{Acknowledgments }
AZ thanks participants at the Simons Program on Causality for inspiring discussions on causal representation learning, even though she still mostly thinks about MDPs. AZ also thanks Dan Malinsky for discussions on graphs and structure at an early stage of this project. 
\bibliography{structure-in-rl,rl}
\bibliographystyle{abbrvnat}

\appendix

\onecolumn

\section{Further Discussion}
\if\forneurips 1 

\textbf{Interpreting Assumption 1:}
We can also specify a corresponding probabilistic model. Suppose $P_a(s_{t+1} \mid s_t) \sim N(\mu_a, \Sigma_a).$ and that ${P_a(s_{t+1} \mid s_t)}$ is partitioned (without loss of generality) as $P_a(s_{t+1}^\supp,s_{t+1}^{\supp_c}  \mid s_{t}^\supp,s_{t}^{\supp_c})$. 
Then by \Cref{asn-reward-relevance} $$P_a(s_{t+1}^\supp  \mid s_{t}^\supp) \overset{D}{=} P_a(s_{t+1}^\supp  \mid s_{t}^\supp,s_{t}^{\supp_c}) \sim N(\mu^\supp_a, \Sigma^{\supp,\supp}_a).$$
where the first equality in distribution follows from the conditional independence restriction of \Cref{asn-reward-relevance} and the parameters of the normal distribution follow since marginal distributions of a jointly normal random variable follow by subsetting the mean vector/covariance matrix appropriately.

\begin{remark}
    Similar to previous works studying similar structures, we assume this structure holds. If it may not,
    we could use model selection methods \citep{lee2022oracle}: if we incorrectly assume this structure, we would obtain a completeness violation; so the model selection method's oracle inequalities would apply and be rate-optimal relative to non-sparse approaches. We emphasize that we don't posit this method as a general alternative to general sparsity, but rather as a simple principled approach to estimate in settings with this exogenous structure. 
\end{remark}
\fi 
\if\forneurips 1 
\paragraph{Why not simply run thresholded LASSO fitted-Q-iteration? }
Lastly, we provide some important motivation by outlining potential failure modes of simply applying thresholded lasso fitted-Q-iteration (without specializing to the endogenous-exogenous structure here). 
The first iteration (last timestep), $q_{T}= R_{T}$. %
So thresholded regression at last timestep is analogous to thresholded reward regression. Note that if reward regression succeeds at time ${T}$, then we are integrating a dense measure against the sparse function $V_{T}$. On the other hand, mistakes in time ${T}$ will get amplified (i.e. upboosted as ``signal" by the dense transition measure). Our reward-thresholded LASSO will not accumulate this error based on the structural assumptions. Without these structural assumptions, it would be unclear whether the rewards are truly dense or whether the dense transitions are amplifying errors in support recovery on the rewards. 
\fi 

\section{Further details on method}
\paragraph{Choosing the penalty level in practice}

A data driven suggestion of \citep{belloni2013least} is to choose $$ \lambda = \frac{c' \hat \sigma \Lambda(1-\alpha\mid X)}{n}$$ where $\Lambda(1-\alpha\mid X)$ is the $(1-\alpha)$ quantile of $n \| S / \sigma \|_\infty$. They also suggest to choose a data-driven upper bound for $\hat\sigma_0$ the sample deviation of $y_i$, compute LASSO, and then set $\hat \sigma^2 = \hat q(\hat \beta)$. 

\section{Proofs}

\subsection{Proofs of characterization}
\begin{proof}[Proof of \Cref{prop-policy-sparsity}]

    The proof follows by induction. We first show the base case, for $t=T$. Recall that we take the convention that $r_{T} = q_{T}=0,$ so that $q_{T}(s,a) = r_{T}(s,a)$. Therefore since $r_{T}(s,a) = r_{T}(\tilde{s},a)$ for $s_{T},\tilde{s}_{T}$ such that $s^\supp_T = \tilde{s}^\supp_{T}$, we also have that for 
    $$\pi^*_{T}(s) \in \arg\max_{a\in\mathcal{A}} r(s,a), \qquad \tilde{\pi}^*_{T}(\tilde{s}) \in \arg\max_{a\in\mathcal{A}} r(\tilde{s},a),$$
    and when $s^\supp_{T} = \tilde{s}^\supp_{T},$
   $\pi^*_{T}(s_{T}) =\tilde{\pi}^*_T(\tilde{s}_{T}).$
   Therefore $q^*_{T}(s_{T},a) = q^*_T(\tilde{s}_{T},a)$ when $s^\supp = \tilde{s}^\supp.$
    Next we show the inductive step. The inductive hypothesis is that $$q^*_{t+1}(s_{t+1},a)=q^*_{t+1}(s_{t+1}^\supp,a) = q^*_{t+1}(\tilde{s}_{t+1},a), \forall a \in \mathcal{A},\text{ and }\pi^*_{t+1}(s_{t+1}) =\tilde{\pi}^*_{t+1}(\tilde{s}_{t+1})\text{ when }s^\supp_{t+1} = \tilde{s}^\supp_{t+1}.$$

    Then for 
    \begin{align*}
        \pi^*_t(s) &\in \arg\max_{a\in\mathcal{A}} \{r_t(s,a)+
        \gamma \E[ q^*_{t+1}(s_{t+1},\pi_{t+1}^*(s_{t+1})) \mid s,a] \}\\
        \tilde{\pi}^*_t(\tilde{s}) &\in \arg\max_{a\in\mathcal{A}}
        \{r_t(\tilde{s},a)+\gamma \E[ q^*_{t+1}(\tilde{s}_{t+1},\pi_{t+1}^*(s_{t+1})) \mid \tilde{s},a]\}
    \end{align*}
we have that 
\begin{align*}q_t^*(s_t,a)   =\; &r_t(s_t,a)+\gamma \E[ q^*_{t+1}(s_{t+1},\pi_{t+1}^*(s_{t+1})) \mid s_t,a] \\
   =\; &  r_t(s_t,a)+\gamma \E[ q^*_{t+1}(s_{t+1}^\supp,\pi_{t+1}^*(s_{t+1}^\supp)) \mid s_t,a] && \tag{induction hypothesis}\\
   =\; &  r_t(s_t,a)+\gamma \E[ q^*_{t+1}(s_{t+1}^\supp,\pi_{t+1}^*(s_{t+1}^\supp)) \mid s^\supp_t,a] && \tag{\Cref{asn-reward-relevance}}\\
   =\;&
r_t(\tilde{s}_t,a)+\gamma \E[ q^*_{t+1}(\tilde{s}_{t+1},\pi_{t+1}^*(\tilde{s}_{t+1})) \mid \tilde{s}^\supp_t,a] \tag{$s^\supp_t = \tilde{s}^\supp_t$}\\
= \;&q_t^*(\tilde{s}_t,a) 
\end{align*}
 when $s^\supp_t = \tilde{s}^\supp_t$.

     Therefore, when $s^\supp_t = \tilde{s}^\supp_t,$
   $$\pi^*_t({s}_t) =\tilde{\pi}^*_t(\tilde{s}_t).$$
   This completes the induction.
\end{proof}

\begin{proof}[Proof of \Cref{prop-bellman-complete}]
   
   Let $\pi^*(s_{t+1}) \in \arg\max_{a \in \mathcal{A}} \breve{q}(s_{t+1},a)$. Note that when $\breve{q}\in\breve{\mathcal{q}}_{t+1},$ the optimal action remains the same for states that differ only outside of the sparse support: $\pi^*(s_{t+1}) = \pi^*(\tilde{s}_{t+1})$ when $s_{t+1}^\supp = \tilde{s}_{t+1}^\supp.$

Therefore for any $\breve{q} \in \breve{\mathcal{q}}_{t+1},$
\begin{align*}
    \mathcal{T}^* \breve{q} & = 
    \E_{s_{t+1}^\supp } \left[\E_{s_{t+1}^{\supp_c}} \left[  \breve{q}_{t+1}(s_{t+1},a^*(s_{t+1})) \mid s_{t+1}^\supp, s,a   \right] \mid s,a \right] \\
  &  = 
    \E_{s_{t+1}^\supp }  \left[  {q}_{t+1}(s_{t+1}^\supp,a^*(s_{t+1}^\supp)) \mid s,a \right] \\
    & = 
    \E_{s_{t+1}^\supp }  \left[  {q}_{t+1}(s_{t+1}^\supp,a^*(s_{t+1}^\supp)) \mid s^\supp,a \right] && \text{ by \Cref{asn-reward-relevance}}
\end{align*}

where the second-to-last equality holds since $\breve{q}_{t+1}(s_{t+1},a^*(s_{t+1}))=\breve{q}_{t+1}(\tilde{s}_{t+1},a^*(s_{t+1}))$ when $s_{t+1}^\supp = \tilde{s}_{t+1}^\supp,$ for any $\breve{q}_{t+1} \in \breve{q}_{t+1}.$ 

Next we show that under \Cref{asn-linear-mdp,asn-reward-relevance}, $\E[r_t(s,a) +  \mathcal{T}^* \breve{q}_{t+1}\mid s,a]$ is linear and is representable by a function $\breve{q} \in \mathcal{q}_{t}.$ Under linear rewards, $r_t(s,a) = \theta_t^* \phi_\supp(s,a)$ for some $\theta_t^*$ that is $\supp$-sparse. And, under linear transitions, $\E_{s_{t+1}^\supp}[\mathcal{T}^* \breve{q}_{t+1} \mid s_t, a_t] = \phi^\top_\supp
\mu_\supp^{*,\top} \breve{q}_{t+1}^* $ where $\mu_\supp^*$ is the $\supp-$marginalized linear transition map.
Hence $$
\E[r_t(s,a) +  \mathcal{T}^* \breve{q}_{t+1}\mid s,a] = 
\underbrace{(\theta_t^* + \breve{q}_{t+1}^{*,\top} \mu_\supp^{*} )}_{w^*_\supp}\phi_\supp(s,a)
$$

 \end{proof}

\subsection{Intermediate results}

We first study the parameter error of ordinary least squares under a missing set of covariates. We let $\covsset$ denote the subset of covariates, for example that returned by thresholded lasso. We first consider the case when $\covsset$ is a given subset containing the true support. The next theorem is a more complex extension, specialized to our reward-thresholded q-estimation setting, of a result about estimation under omitted variables of \citep{zhou2010thresholded,zhou2009thresholding}. The key structure allowing us to link thresholded lasso of reward to prediction error of estimated $q$ functions is the shared covariance structure. \Cref{thm-q-covsset} is the main technical contribution of our work. 
\begin{theorem}[Prediction error bounds of $\covsset$-restricted ordinary least squares of the Bellman residual]\label{thm-q-covsset}
Suppose \Cref{asn-block-indep,asn-linear-mdp,asn-rec,asn-reward-relevance,asn-time-homogeneous}. Let $\mathcal{D} = \{ 1, \dots, d\} \setminus \covsset$ and $\supp_{\mathcal{D}} = \mathcal{D} \bigcap \supp$ (e.g. the set of false negatives of support recovery). Suppose $\vert \supp \bigcup \supp_{\mathcal{D}}\vert \leq 2\sparscard$ and that $\supp \subseteq \covsset.$ 
Suppose that $\lambda \geq \norm{\frac{X^\top \epsilon_{r+q} }{\sqrt{n} }}_\infty$.
Consider $\covsset-$restricted ordinary least squares regression of sparse $q$. In the following, we omit the time index for brevity. Then: 
\begin{align*} 
\norm{ \hat\theta_{\covsset}-\theta_{\covsset}}_2 &\leq 
\frac{\sqrt{\vert \covsset \vert}}{ \Lambda_{\min}(\vert \covsset \vert) } \lambda.
+ \norm{\theta_{\mathcal{D}}}_2 
\\
  \frac 1n  \norm{X \hat{\theta}-X \theta^* }_2^2 & \leq 
4 \frac{\sigma^2_q (\vert \covsset \vert+2 \sqrt{ \vert \covsset \vert\log (1 / \delta)}+ \log (1 / \delta))}{n} + 
\max \left(36 \frac{\vert \covsset \vert \lambda^2}{\kappa }, 162 \frac{\sigma^2_\theta |\covsset| \log (d /|\covsset|)}{n}
\right) 
\end{align*} 
\end{theorem}

\begin{proof}[Proof of \Cref{thm-q-covsset}]

Let $\theta$ be the full ordinary least squares solution for the $q$ estimation, $\hat \theta_{\covsset}, \theta_{\covsset}^*$ be the estiamted and true restricted OLS solution computed on $\covsset$, respectively, and $\theta^*$ the true (sparse) solution. 
\begin{align*}
\hat\theta_{\covsset} &= (X^\top_{\covsset}X_{\covsset})^{-1} X^\top_{\covsset}Y(q) = 
(X^\top_{\covsset}X_{\covsset})^{-1} X^\top_{\covsset}Y(q)  \\
& = (X^\top_{\covsset}X_{\covsset})^{-1} X^\top_{\covsset}(X^\top_{\covsset} \theta_{\covsset}^* + \epsilon )  
\tag{ sparse rewards and asn. about sparse $Q$ function }\\
& = \theta_{\covsset}^* + (X^\top_{\covsset}X_{\covsset})^{-1} X^\top_{\covsset}\epsilon_{r+\gamma q}
\end{align*} 
Hence, 
$$ \norm{ \hat\theta_{\covsset}-\theta_{\covsset}^*}_2 \leq \norm{(X^\top_{\covsset}X_{\covsset})^{-1} X^\top_{\covsset}\epsilon_{r+\gamma q} }_2. 
$$
We bound the second term as follows: 

$$\norm{(X^\top_{\covsset}X_{\covsset})^{-1} X^\top_{\covsset}\epsilon_{r+\gamma q} }_2
\leq \lrnorm{ \left( \frac{X^\top_{\covsset}X_{\covsset} }{n}\right)^{-1} }_2 \lrnorm{\frac{X_{\covsset}^\top \epsilon_{r+\gamma q}}{n}}_2 \leq \frac{\sqrt{\vert \covsset \vert}}{ \Lambda_{\min}s(\vert \covsset \vert) } \lambda,$$
yielding that
$$
\norm{ \hat\theta_{\covsset}-\theta_{\covsset}^*}_2 
\leq 
\frac{\sqrt{\vert \covsset \vert}}{ \Lambda_{\min}(\vert \covsset \vert) } \lambda.
$$

The result follows since $\norm{ \hat\theta_{\covsset}-\theta^*}_2^2 \leq 
2\norm{ \hat\theta_{\covsset}-\theta_{\covsset}^*}_2^2 + 
2\norm{ \theta_{\covsset}^*-\theta^*}_2^2.
$

Next we bound the prediction error, $$  \frac 1n  \norm{X \hat{\theta}-X \theta^* }_2^2 
\leq \frac 1n  \norm{X \hat{\theta}-X \theta^*_{\covsset} }_2^2
+ \frac 1n  \norm{X \theta^*_{\covsset}-X \theta^* }_2^2
.$$ 

We will decompose relative to $\hat\beta$, a thresholded lasso regression on rewards $r$ alone but also restricted to $\covsset$. %
Note that by the elementary bound $(a-b)b \leq (a-b)^2 +b^2$: 
\begin{align*}
  \frac 1n \norm{X \hat{\theta}-X \theta_{\covsset}^* }_2^2
    &\leq    \frac 2n \norm{X \hat{\theta}-X \theta_{\covsset}^* - (X \hat{\beta}-X \beta^*_{\covsset} )  }_2^2+ \frac 2n\norm{ X \hat{\beta}-X \beta^*_{\covsset}  }_2^2 \\
    & \coloneqq T_1 + T_2. 
\end{align*}

First we bound $T_1 \coloneqq \frac 2n \norm{X \hat{\theta}-X \theta_{\covsset}^* - (X \hat{\beta}-X \beta^*_{\covsset} )  }_2^2$. Observe that it is equivalently the prediction error when regressing the next-stage $q$ function alone, i.e. $y_t - r_t$, on the $\covsset$-restricted features, since $(\hat{\theta}-  \hat{\beta}) = (X^\top_{\covsset} X_{\covsset})^{-1} X^\top_{\covsset} \{ \gamma v(s') \}$. Then 
$$
\frac 2n
\left\|X (\hat{\theta}-  \hat{\beta})-X (\theta_{\covsset}^*-\beta^*_{\covsset}  )  \right\|_2^2 = 
\frac 2n\left\|X 
\left\{ (X_I^\top X_I)^{-1} X^\top_I 
 ( \gamma v(s_{t+1}))
 - (\theta_{\covsset}^*-\beta^*_{\covsset}  )) 
\right\}\right\|_2^2,
$$

where the last term can be identified as the noise term in $(V(s_{t+1}))- (\theta_{\covsset}^*-\beta^*_{\covsset}  )) \approx \epsilon_q$ under the linear MDP assumption. By the sparsity properties of $\hat\theta,\hat\beta$ (they are both restricted to $\covsset$): 
\begin{align*}
\frac 2n\left\|X (\hat{\theta}-  \hat{\beta})-X (\theta_{\covsset}^*-\beta^*_{\covsset}  )  \right\|_2^2
&= 
\frac 2n\left\|X_{\covsset} (\hat{\theta}-  \hat{\beta})-X_{\covsset} (\theta_{\covsset}^*-\beta^*_{\covsset}  )  \right\|_2^2 \tag{by two-step procedure and realizability}\\
& \leq 
 \frac{\sigma^2_q (2\vert \covsset \vert +2 \sqrt{2 \vert \covsset \vert \log (1 / \delta)}+2 \log (1 / \delta))}{n}. \tag{ by \Cref{lemma-randdesign-leastsq-mse}}
\end{align*}

Next we bound $T_2 \coloneqq \frac 2n \norm{ X \hat{\beta}-X \beta^*_{\covsset} }_2^2$. Let $\beta^\lambda$ denote the initial LASSO solution in the thresholded lasso $\hat\beta$.

By optimality of $\hat \beta$, 
\begin{align*}
 \frac 2n\norm{ X \hat{\beta}-X \beta^*_{\covsset} }_2^2 
 & \leq \frac 2n\norm{ X {\beta}^\lambda-X \beta^*_{\covsset} }_2^2 \\
 &\leq 
 \frac 4n \norm{ X {\beta}^{\lambda}-X \beta^* }_2^2 +   \frac 4n \norm{ X \beta^* - X \beta_{\covsset}^* }_2^2
\end{align*}
The first of these is bounded via standard analysis of prediction error in LASSO, and the second by a maximal inequality as previously.

By the penalized formulation: 
\begin{align*}
\frac{1}{2n} \left\|X {\beta^\lambda}-X \beta^*\right\|_2^2 & \leq \frac{\lambda}{2}\left\|{\beta^\lambda}-\beta^*\right\|_1+\lambda\left(\left\|\beta^*\right\|_1-\|{\beta^\lambda}\|_1\right) \\ & \leq \frac{\lambda}{2}\left\|{\beta^\lambda_{\covsset}}-\beta_{{\covsset}}^*\right\|_1+\lambda\left\|{\beta^\lambda_{{\covsset_c}}}\right\|_1+\lambda\left(\left\|\beta^*\right\|_1-\|{\beta^\lambda}\|_1\right) \\ & \leq \frac{\lambda}{2}\left\|{\beta^\lambda_{\covsset}}-\beta_{{\covsset}}^*\right\|_1+\lambda\left\|{\beta^\lambda_{{\covsset_c}}}\right\|_1+\lambda\left(\left\|\beta_{{\covsset}}^*-{\beta^\lambda_{\covsset}}\right\|_1-\left\|{\beta^\lambda_{{\covsset_c}}}\right\|_1\right) \\ & =\frac{3 \lambda}{2}\left\|{\beta^\lambda_{\covsset}}-\beta_{{\covsset}}^*\right\|_1-\frac{\lambda}{2}\left\|{\beta^\lambda_{{\covsset_c}}}\right\|_1,
\end{align*}
The above, with the restricted eigenvalue condition of \Cref{asn-rec}, implies that 
\begin{equation}
    \frac{3 \lambda}{2}\left\|{\beta^\lambda_{\covsset}}-\beta_{{\covsset}}^*\right\|_1-\frac{\lambda}{2}\left\|{\beta^\lambda_{{\covsset_c}}}\right\|_1
\geq \frac{1}{2n} \norm{ X {\beta^\lambda}-X \beta^* }_2^2 \geq \kappa \norm{{\beta^\lambda}-\beta^*}_2^2  \label{eq-apx-lassoprederror}
\end{equation}

Therefore, by properties of the $\ell_1$ and $\ell_2$ norm: 
$$ \frac{1}{2n} \norm{ X {\beta^\lambda}-X \beta^* }_2^2
\leq \frac{3 \lambda}{2} \norm{\beta^\lambda_{\covsset}-\beta_{{\covsset}}^*}_1 \leq  \frac{3 \lambda \sqrt{\vert \covsset \vert} }{2} \norm{\beta^\lambda_{\covsset}-\beta_{{\covsset}}^*}_2   $$
Then applying the restricted eigenvalue condition of \Cref{asn-rec} to the last term of the above, we obtain that 
$$
\frac{1}{2n} \left\|X {\beta^\lambda}-X \beta^*\right\|_2^2 
\leq
\frac{3 \lambda \sqrt{\sparscard} \left\|X {\beta^\lambda}-X \beta^*\right\|_2}{ \sqrt{n \kappa}}.
$$

Rearranging, this gives the bound 
\begin{align*}\frac{4}{n}\left\|X {\beta^\lambda}-X \beta^*\right\|_2^2 \leq 144 \frac{\vert \covsset \vert \lambda^2}{\kappa }.
\end{align*}

Finally, we can bound $\frac 2n\norm{ X \beta_{\covsset}^*-X \beta^* }_2^2$ via a maximal inequality over the $\ell_0$ norm ball of radius $2 \oracsparscard$ since earlier we showed that $\vert \covsset \vert \leq 2 \oracsparscard$. Applying the maximal inequality of \Cref{lemma-maximalinequality-l0ball} gives 
$$
\frac{4}{n} \norm{ X \beta_{\covsset}^*-X \beta^* }_2^2
\leq 
324 \frac{\sigma^2_\theta |\covsset| \log (d /|\covsset|)}{n}.
$$

\end{proof} 

\subsection{Proofs of main results for method}

\begin{proof}[Proof of \Cref{thm-q-predictionerror}]

Because the support is recovered from a thresholded LASSO on the rewards, the support inclusion result is a consequence of \citep[Thm. 6.3]{zhou2010thresholded}, although analogous results essentially hold under stronger beta-min conditions (i.e, on the support $\supp$ and correspondingly  stronger support inclusion conditions). 
Namely, it gives that, suppose for some constants $\breve{D}_1 \geq D_1$, and for $D_0, D_1$ such that: 
For $K:=\kappa\left(\oracsparscard, 6\right)$, $b_0\geq 2,$
\begin{align*}
& D_0=\max \left\{D, K \sqrt{2}\left(2 \sqrt{\Lambda_{\max }\left(\sparscard-\oracsparscard\right)}+3 b_0 K\right)\right\} \\
& \text { where } 
D=(\sqrt{2}+1) \frac{\sqrt{\Lambda_{\max }\left(\sparscard-\oracsparscard\right)}}{\sqrt{\Lambda_{\min }\left(2 \oracsparscard\right)}}+\frac{\theta_{\oracsparscard, 2 \oracsparscard} \Lambda_{\max }\left(\sparscard-\oracsparscard\right)}{\Lambda_{\min }\left(2 \oracsparscard\right)} 
\text { and } \\
& D_1=2 \Lambda_{\max }\left(\sparscard-\oracsparscard
\right) / b_0+9 K^2 b_0 / 2,
\end{align*}

it holds that, for $\breve{D}_1 \geq D_1,$
$$
\beta_{\min , A_0} \geq D_0 \lambda \sigma \sqrt{\oracsparscard}+\breve{D}_1 \lambda \sigma, \text { where } \lambda:=\sqrt{2 \log p / n}.
$$
Choose a thresholding parameter $\threshold$ and set
$$
\covsset=\left\{j:\left|\beta_{j, \text {init}}\right| \geq \threshold\right\}, \text { where } \threshold \geq \breve{D}_1 \lambda \sigma .
$$
Then on $\mathcal{E}_a$, 
\begin{align}
    &\activeset \subset \covsset, 
    \left|\covsset \cap \largestszerocoords^c\right| \leq \oracsparscard, |\covsset| \leq 2\oracsparscard,
\\
&
\left\|\beta_{\mathcal{D}}\right\|_2^2 \leq\left(\oracsparscard-a_0\right) \lambda^2 \sigma^2.
\end{align}

This yields the first statement about support recovery. 

For prediction error, we then apply \Cref{thm-q-predictionerror} and this yields the result. \end{proof}

\begin{proof}[Proof of \Cref{prop-apx-completeness}]
True $\beta^*$ is $\rho$-sparse but the worst case situation is if $\supp \setminus \activeset \not\subseteq \mathcal{I},$ i.e. the low-signal coefficients are not returned by the thresholding algorithm. On the other hand, they are assuredly of magnitude $\leq \vert \lambda \sigma \vert$ and hence ought to lead to less violation of the completeness condition. Let $\mathcal{Q}_{\covsset, \rho\setminus\activeset \not\subseteq \covsset}$ denote the set of linear coefficients with support on $\norm{\covsset}_0 \leq 2 \oracsparscard$ such that it does not contain the low signal variables $\rho\setminus\activeset,$ and 
     $$ \textstyle \underset{q_{t+1} \in \mathcal{Q}_{\covsset, \rho\setminus\activeset \not\subseteq \covsset}}{\sup} \;\; \underset{q_t \in \mathcal{Q}_{\covsset, \rho\setminus\activeset \not\subseteq \covsset}}{\inf}
     \|q_t-\mathcal{T}_t^{\star} q_{t+1}\|_{\mu_t}^2 \leq \epsilon.$$
The infimum over $q_t$ is equivalent to a further-restricted $\ell_0$ norm regression problem. 
\begin{align*}
    &  \underset{q_{t+1} \in \mathcal{Q}_{\covsset, \rho\setminus\activeset \not\subseteq \covsset}}{\sup} \;\;\underset{q_t \in \mathcal{Q}_{\covsset, \rho\setminus\activeset \not\subseteq \covsset}}{\inf}
     \|q_t-\mathcal{T}_t^{\star} q_{t+1}\|_{\mu_t}^2 \\
     & = \underset{q_t \in \mathcal{Q}_{\covsset, \rho\setminus\activeset \not\subseteq \covsset}}{\inf}  \;\;\underset{q_{t+1} \in \mathcal{Q}_{\covsset, \rho\setminus\activeset \not\subseteq \covsset}}{\sup} 
     \|q_t-\mathcal{T}_t^{\star} q_{t+1}\|_{\mu_t}^2
\end{align*}
and 
\begin{align*}
&\underset{q_{t+1} \in \mathcal{Q}_{\covsset, \rho\setminus\activeset \not\subseteq \covsset}}{\sup} 
    \{ \|q_t-\mathcal{T}_t^{\star} q_{t+1}\|_{\mu_t}^2+
     \|q_t-\mathcal{T}_t^{\star} q_{t+1}\|_{\mu_t}^2 \}\\
     &
 \leq 
 \underset{q_{t+1} \in \mathcal{Q}_{\covsset, \rho\setminus\activeset \not\subseteq \covsset}}{\sup} 
     \|q_t-\mathcal{T}_t^{\star} q_{t+1}\|_{\mu_t}^2
 + \underset{q_{t+1} \in \mathcal{Q}_{\covsset, \rho\setminus\activeset \not\subseteq \covsset}}{\sup} 
\|\mathcal{T}_t^{\star}q^*_{t+1}-\mathcal{T}_t^{\star} q_{t+1}\|_{\mu_t}^2
\end{align*} 
Then $$\underset{q_{t+1} \in \mathcal{Q}_{\covsset, \rho\setminus\activeset \not\subseteq \covsset}}{\sup} 
\|\mathcal{T}_t^{\star}q^*_{t+1}-\mathcal{T}_t^{\star} q_{t+1}\|_{\mu_t}^2 \leq (\norm{(q_{t+1})_{\covsset \setminus \activeset}}_1 + 
\norm{(q_{t+1})_{\supp \setminus \activeset}}_1 )^2\leq (2 s \threshold + \sqrt{2s} \norm{\hat\beta - \beta}_2 + s \lambda \sigma)^2$$
That is, false positives are of low signal strength (by the algorithm, and by prediction error bound) while false negatives not in the active set are also of low signal strength. The threshold and signal strength definitions tend to $0$ at a rate overall depending on $\lambda$.
Therefore, using a (loose) bound that $(a+b)^2 \leq 2a^2 + 2b^2$, 
$$
\underset{q_{t+1} \in 
\mathcal{Q}_{\covsset, \rho\setminus\activeset \not\subseteq \covsset}
}{\sup} 
\|\mathcal{T}_t^{\star}q^*_{t+1}-\mathcal{T}_t^{\star} q_{t+1}\|_{\mu_t}^2
\leq (2 s (\threshold+\lambda \sigma) + \sqrt{2s} \norm{\hat\beta - \beta}_2 )^2 \leq 16s^2 
(\threshold^2+\lambda^2 \sigma^2) + 
4s \norm{\hat\beta - \beta}_2^2
$$

Next we bound: $$\underset{q_t \in \mathcal{Q}_{\covsset, \rho\setminus\activeset \not\subseteq \covsset}}{\inf}
\;\; \underset{q_{t+1} \in \mathcal{Q}_{\covsset, \rho\setminus\activeset \not\subseteq \covsset}}{\sup} 
     \|q_t-\mathcal{T}_t^{\star} q_{t+1}\|_{\mu_t}^2$$

    The outer minimization is simply least-squares regression over a further restricted $\ell_0$ norm ball. Consider $\tilde{\mathcal{Q}}$ such that $\tilde{\mathcal{Q}} = \{ q \in \mathcal{Q}_{\covsset, \rho\setminus\activeset \not\subseteq \covsset} \colon 
    q_{\activeset} > 0,q_{\covsset\setminus \activeset} = 0 \}$, and note that $\tilde{\mathcal{Q}} \subseteq \mathcal{Q}_{\covsset, \rho\setminus\activeset \not\subseteq \covsset}$. 

 $$\underset{q_t \in \mathcal{Q}_{\covsset, \rho\setminus\activeset \not\subseteq \covsset}}{\inf}
\;\; \underset{q_{t+1} \in \mathcal{Q}_{\covsset, \rho\setminus\activeset \not\subseteq \covsset}}{\sup} 
     \|q_t-\mathcal{T}_t^{\star} q_{t+1}\|_{\mu_t}^2 \leq 
     \underset{q_t \in \tilde{\mathcal{Q}}}{\inf}
\;\; \underset{q_{t+1} \in \mathcal{Q}_{\covsset, \rho\setminus\activeset \not\subseteq \covsset}}{\sup} 
     \|q_t-\mathcal{T}_t^{\star} q_{t+1}\|_{\mu_t}^2 
     $$
The worst-case error is incurred when $q_{t+1,\supp \setminus \activeset}>0$; these are the low-signal variables not guaranteed to be recovered by the algorithm. Then for $q' \in \mathcal{Q}_{\setminus \activeset} \coloneqq 
\{ q \in \mathcal{Q}_{\covsset, \rho\setminus\activeset \not\subseteq \covsset} 
\colon 
q_{\supp \setminus \activeset}>0
\},$
and we have that 
 $$ \leq 
     \underset{q_t \in \tilde{\mathcal{Q}}}{\inf}
\;\; \underset{q_{t+1} \in \mathcal{Q}_{\setminus \activeset} 
}{\sup} 
     \|q_t-\mathcal{T}_t^{\star} q_{t+1}\|_{\mu_t}^2, 
     $$
where the leading order dependence is described by \Cref{thm-q-covsset}'s analysis of least-squares regression on a restricted covariate set: $\tilde{\mathcal{Q}}$ omits the low-signal variables $\supp \setminus \activeset.$ Therefore, by \Cref{thm-q-covsset}, w.h.p. under $\mathcal{E}_a$ and assumptions on $\lambda$ in \Cref{thm-q-covsset},
$$
\underset{q_{t+1} \in \mathcal{Q}_{\covsset, \rho\setminus\activeset \not\subseteq \covsset}}{\sup}\;\; \underset{q_t \in \mathcal{Q}_{\covsset, \rho\setminus\activeset \not\subseteq \covsset}}{\inf}
     \|q_t-\mathcal{T}_t^{\star} q_{t+1}\|_{\mu_t}^2
= O_p(n^{-1}).
     $$

\end{proof}

\subsection{Technical results}
We list standard technical results from other works that we use without proof. 
\subsubsection{Concentration}

\begin{lemma}[Theorem 1 of \citep{hsu2011analysis}, random design prediction bound for linear regression. ]\label{lemma-randdesign-leastsq-mse}
Define $\widehat{\Sigma}:=\widehat{\mathbb{E}}[x \otimes x]=\frac{1}{n} \sum_{i=1}^n x_i \otimes x_i$. Suppose outcomes are $\sigma_{noise}$-subgaussian and ``bounded statistical leverage'', then there exists a finite $\rho_{2 \text {,cov }} \geq 1$ such that almost surely:
$$
\frac{\left\|\Sigma^{-1 / 2} X\right\|}{\sqrt{d}}=\frac{\|\Sigma^{-1 / 2} X\|}{\sqrt{\mathbb{E}[\|\Sigma^{-1 / 2} X\|^2]}} \leq \rho_{2, \mathrm{cov}}
$$   
If $n>n_{2, \delta}$, then with probability at least $1-2 \delta$, we have
that the matrix error 
$
\left\|\Sigma^{1 / 2} \hat{\Sigma}^{-1} \Sigma^{1 / 2}\right\| \leq K_{2, \delta, n} \leq 5 ;
$
and the excess loss satisfies:
$$
\left\|\hat{w}_{\mathrm{ols}}-w\right\|_{\Sigma}^2 \leq K_{2, \delta, n} \cdot \frac{\sigma_{\text {noise }}^2 \cdot(d+2 \sqrt{d \log (1 / \delta)}+2 \log (1 / \delta))}{n}
$$
\end{lemma}

\begin{lemma}[Prediction error bounds via maximal inequalities over an $\ell_0$ ball, Theorem 4 of \citep{raskutti2011minimax} .]\label{lemma-maximalinequality-l0ball}
For any covariate matrix $X$, with probability greater than $1-\exp (-10 s \log (d / s))$ the minimax prediction risk is upper bounded as
$$
\min_{\widehat{w}} \max _{w^* \in \mathbb{B}_0(\vert \covsset \vert )} \frac{1}{n}\left\|X\left(\widehat{w}-w^*\right)\right\|_2^2 \leq 81 \frac{\sigma^2 \vert \covsset \vert \log (d / \vert \covsset \vert)}{n},
$$
where $\mathbb{B}_0(\vert \covsset \vert)$ is the $\ell_0$ norm ball of radius $\vert \covsset \vert.$
\end{lemma}

\subsubsection{Analysis of fitted-Q-evaluation}

\begin{definition}[Bellman error]\label{def-bellmanerror}
    Under data distribution $\mu_t$, define the Bellman error of function $q = (q_0, \dots, q_{T-1})$ as: $\textstyle
\mathcal{E}(q) = \frac{1}{T} \sum_{t=0}^{T-1} \norm{q_t - \mathcal{T}^*_t q_{t+1}}_{\mu_t}^2$
\end{definition} 
\begin{lemma}[Bellman error to value suboptimality]\label{lemma-fqi-bmanerrortovaluesuboptimality}
    Under \Cref{asn-rec}, for any $q \in \mathcal{Q},$ we have that, for $\pi$ the policy that is greedy with respect to $q,$ $ V_1^*(s_1) - V_1^\pi(s_1) \leq 2T \sqrt{C \cdot \mathcal{E}(q^\pi) }.$
\end{lemma}

\begin{proof}[Proof of \Cref{thm-fqi-bound}]
Under \Cref{lemma-fqi-bmanerrortovaluesuboptimality}, it suffices to bound the Bellman error, $\frac{1}{T} \sum_{t=0}^{T-1} \norm{q_t - \mathcal{T}^*_t q_{t+1}}_{\mu_t}^2$. 
We start with one timestep. Let $\ell(f,g) = (f-g)^2$ be the squared error. The Bellman error satisfies that $\|\hat{q}_h-\mathcal{T}_h^{\star} \hat{q}_{h+1}\|_{\mu_h}^2$ and can be lower bounded as follows: 
\begin{align*}\|\hat{q}_h-\mathcal{T}_h^{\star} \hat{q}_{h+1}\|_{\mu_h}^2
&=\mathbb{E}_{\mu_h} [\ell(\hat{q}_h, \hat{q}_{h+1})]-\mathbb{E}_{\mu_h} [\ell(q_h^{\dagger}, \hat{q}_{h+1})]+\|q_h^{\dagger}-\mathcal{T}_h^{\star} \hat{q}_{h+1}\|_{\mu_h}^2 \\
&\leq \mathbb{E}_{\mu_h} [\ell(\hat{q}_h, \hat{q}_{h+1})] + \epsilon \tag{by \Cref{prop-apx-completeness} on apx. Bellman completeness}
\end{align*}
where $\epsilon$ is the parameter for approximate Bellman completeness, such that $${\underset{q_{t+1} \in \mathcal{Q}_{\covsset, \rho\setminus\activeset \not\subseteq \covsset}}{\sup} \underset{q_t \in \mathcal{Q}_{\covsset, \rho\setminus\activeset \not\subseteq \covsset}}{\inf}
     \|q_t-\mathcal{T}_t^{\star} q_{t+1}\|_{\mu_t}^2 \leq \epsilon}.$$ By \Cref{prop-bellman-complete} we have that $\epsilon = O_p(n^{- 1})$. 

The prediction error bound of \Cref{thm-q-predictionerror} bounds $\mathbb{E}_{\mu_h} [\ell(\hat{q}_h, \hat{q}_{h+1})]$ so we have that 
$$
V_1^*(s_1) - V_1^\pi(s_1) \leq  2T \sqrt{\frac{\concentratabilitycoef \sigma^2_q ( 2 \sparscard (1 + 468 \log (2d)) +2 ( 1 + 2 \sqrt{\sparscard}) 
 }{n}. }
$$

\end{proof}

\section{Alternative model: endogenous/exogenous decomposition of \cite{dietterich2018discovering}}\label{apx-sec-exoendo}

We discuss a related, but different model: a sparse reward variant of the endogenous-exogenous variable decomposition of \citet{dietterich2018discovering}. The main difference is that the exogenous components instead can affect the endogenous components, as opposed to the other way around in our model, where endogenous components affect exogenous components. We include the illustration in \Cref{fig:exo-endo}.

A natural question is whether our methods can handle this setting as well, especially since \citet{dietterich2018discovering} shows that the optimal policy is sparse in the endogenous MDP alone. Our exact characterization in this paper used the conditional independence restriction of \Cref{asn-reward-relevance}, which does not hold in the exo-endo MDP since exogenous variables can affect next-timestep endogenous variables. On the other hand, that the optimal policy is sparse in the endogenous MDP alone implies that the corresponding \textit{advantage functions}, i.e. $\adv_{a_0}(s,a) = q(s,a) - q(s,a_0)$ do in fact satisfy the conditional independence restriction of \Cref{asn-reward-relevance}. 

Hence, under the additional restriction of reward sparsity where exogenous variables do not affect reward, we can extend methods in this paper to thresholded-LASSO based on estimating reward \textit{contrast} functions and hence advantage functions. To sketch this extension, note that we can run CATE estimation at the final timestep and then simply redefine Bellman targets to be differences of q-functions over actions. 

This additional assumption of reward sparsity is required: in the original paper of \citep{dietterich2018discovering}, rewards are additively decomposable but there can be direct effect of exogenous variables on the reward.

\begin{figure}
    \centering
\begin{tikzpicture}
    \node[state] (s) at (0,0) {$s^\supp_0$};
    \node[state] (x) [above =of s] {$s^{\supp_c}_0$};
    \node[state] (a) [right =of s] {$a_0$};
    \node[state] (r) [below =of a] {$r_0$};
    \node[state] (s1) [right =of a] {$s^{\supp}_1$};
    \node[state] (a1) [right =of s1] {$a_1$};
    \node[state] (r1) [below =of a1] {$r_1$};
    \node[state] (x1) [above =of s1] {$s^{\supp_c}_1$};
    \path (s) edge (r);
    \path (s) edge (a);
    \path (a) edge (s1);
    \path (a) edge (r);
    \path (x) edge (a);
    \path (x) edge (x1);
    \path (x) edge (s1);
        \path (s) edge[bend right=30] (s1);
    \path (s1) edge (a1);
\path (s1) edge (r1);
\path (a1) edge (r1);
    \path (x1) edge (a1);
\end{tikzpicture}
    \caption{``Exogenous/endogenous MDP" of \citep{dietterich2018discovering}.} 
    \label{fig:exo-endo}
\end{figure}

\section{Experiments}

In the data-generating process, we first consider $\vert \mathcal{S} \vert = 50, \sparscard = 10$, and $\mathcal{A} = \{ 0, 1\}$. The reward and states evolve according to $$r_t(s,a) = \beta^\top \phi_t(s,a) + \epsilon_r, \;\; s_{t+1}(s,a) = M_a s + \epsilon_s.$$ Recalling that $M_a = \begin{bmatrix} M_{a}^{\supp \to \supp } & 0 \\
M_{a}^{\supp \to {\supp_c} } & M_{a}^{{\supp_c} \to {\supp_c} } 
\end{bmatrix},$ we generate the coefficient matrix with independent  normal random variables $\sim N(0.2, 1)$. (Note that the nonzero mean helps ensure the beta-min condition). The zero-mean noise terms are normally distributed with standard deviations $\sigma_s = 0.4, \sigma_r = 0.6.$ In the estimation, we let $\phi(s,a)$ be a product space over actions, i.e. equivalent to fitting a $q$ function separately for every action. 

We first show experiments for policy evaluation in the main text due to space constraints. Fitted-Q-evaluation is similar to fitted-Q-iteration, but replaces the max over q functions with the expectation over actions according to the next time-step's policy. See the appendix for additional experiments for policy optimization specifically. We compare our reward-filtered estimation using \Cref{alg-rewardsparse-fqi} with naive thresholded lasso, i.e. thresholding lasso-based estimation of q-functions alone in \Cref{subfig:fpr,subfig:mse,subfig:tpr}. (We average the $q$ function over actions; results are similar across actions). The behavior and evaluation policies are both (different) logistic probability models in the state variable, with the coefficient vector given by (different) random draws from the uniform distribution on $[-0.5,0.5].$ We average over 50 replications from this data generating process and add standard errors, shaded, on the plot. The first plot, \Cref{subfig:mse}, shows the benefits in mean-squared error estimation of the q-function $q_1^{pi_e}(s,a),$ relative to the oracle $q$ function, which is estimated from a separate dataset of $n=20000$ trajectories. The reward-filtered method achieves an order of magnitude smaller mean-squared error for small sample sizes, with consistent improvement over thresholded LASSO estimation on the $q$ function alone. Next in \Cref{subfig:tpr} we show the true positive rate: both methods perform similarly in including the sparse component the recovered support. But the last plot of \Cref{subfig:fpr} shows that the naive thresholded lasso method includes many exogenous variables that are not necessary to recover the optimal policy, while the false positive rate for the reward-filtered method is controlled throughout as a constant fraction of the sparsity. Overall this simple simulation shows the improvements in estimation of the $q$ function (which translate down the line to improvements in decision-value) under this special structure.  

\end{document}